\documentclass[11pt,a4paper]{article}

\usepackage[margin=1in]{geometry}
\usepackage{libertinus}          % elegant serif + math
\usepackage{microtype}
\usepackage{parskip}
\usepackage{xcolor}
\definecolor{Accent}{HTML}{4F81BD}
\usepackage{sectsty}
\allsectionsfont{\color{Accent}\scshape}

\usepackage{enumitem}
\setlist{nosep,leftmargin=*}

\usepackage[hidelinks]{hyperref}
\usepackage{orcidlink}

\usepackage{comment} 
\usepackage{caption}
\usepackage{amsmath}
\usepackage{float}
\usepackage[utf8x]{inputenc}
\usepackage{amsmath,amssymb}
\usepackage{xcolor}
\usepackage{mathtools}
\usepackage{tabulary}
\usepackage{amsthm}
\usepackage{latexsym}
\usepackage[mathcal]{eucal}
\usepackage{amscd}
\usepackage{graphicx}
\usepackage{amsfonts}
\usepackage{amscd}
\usepackage{MnSymbol}
\usepackage{setspace}
\newtheorem{definition}{Definition}[section]
\usepackage[english]{babel}
\usepackage{listings}
\usepackage{calligra}
\usepackage{titlesec}
\usepackage{booktabs}

\newtheorem{remark}{Remark}[section]
\newtheorem{theorem}{Theorem}[section]

\newtheorem{proposition}{Proposition}[section]

\theoremstyle{remark}

\newcommand{\N}{\mathbb{N}}   % Natural numbers: 0,1,2,…
\newcommand{\R}{\mathbb{R}}   % Real numbers

\title{\textbf{\Large Contraction, Criticality, and Capacity: A Dynamical-Systems Perspective on Echo-State Networks}}

\author{
    \large 
    Pradeep Singh\orcidlink{0000-0002-5372-3355}\thanks{Email: \texttt{pradeep.cs@sric.iitr.ac.in}},     
    Lavanya Sankaranarayanan\orcidlink{0009-0004-8939-572X}\thanks{Email: \texttt{lsankara@ualberta.ca}},  Balasubramanian Raman\orcidlink{0000-0001-6277-6267}\thanks{Email: \texttt{bala@cs.iitr.ac.in}},\vspace{0.2cm}\\
    \begin{minipage}[t]{0.5\textwidth}
    \centering
    \small Department of Computer Science and Engineering\\
    \small Indian Institute of Technology Roorkee\\
    \small Roorkee-247667, India
    \end{minipage}
       \begin{minipage}[t]{0.5\textwidth}
    \centering
    \small Department of Mathematical and Statistical Sciences\\
    \small University of Alberta\\
    \small Edmonton, AB T6G 2J5, Canada
    \end{minipage}%
}

\date{}

\begin{document}\maketitle

\begin{abstract}
Echo-State Networks (ESNs) distil a key neurobiological insight: richly recurrent but \emph{fixed} circuitry combined with adaptive linear read-outs can transform temporal streams with remarkable efficiency.  Yet fundamental questions about stability, memory and expressive power remain fragmented across disciplines.  We present a unified, dynamical-systems treatment that weaves together functional analysis, random attractor theory and recent neuroscientific findings.  First, on compact multivariate input alphabets we prove that the \emph{Echo-State Property} (wash-out of initial conditions) together with global Lipschitz dynamics \emph{necessarily} yields the \emph{Fading-Memory Property} (geometric forgetting of remote inputs).  Tight algebraic tests translate activation-specific Lipschitz constants into certified spectral-norm bounds, covering both saturating and rectifying nonlinearities.  Second, employing a Stone–Weierstrass strategy we give a streamlined proof that ESNs with polynomial reservoirs and linear read-outs are dense in the Banach space of causal, time-invariant fading-memory filters, extending universality to stochastic inputs.  Third, we quantify computational resources via Jaeger’s memory-capacity spectrum, show how topology and leak rate redistribute delay-specific capacities, and link these trade-offs to Lyapunov spectra at the “edge of chaos.”  Finally, casting ESNs as skew-product random dynamical systems, we establish existence of singleton pullback attractors and derive conditional Lyapunov bounds, providing a rigorous analogue to cortical criticality.  The analysis yields concrete design rules—spectral radius, input gain, activation choice—grounded simultaneously in mathematics and neuroscience, and clarifies why modest-sized reservoirs often rival fully trained recurrent networks in practice.
\end{abstract}

\section{Introduction}
\label{sec:intro}

\noindent\paragraph{Brains as natural reservoirs.}  
Cortical micro-circuits are dense, heterogeneous and massively
recurrent; action potentials reverberate through loops whose synaptic
time-constants range from a few milliseconds in AMPA-dominated shafts to
hundreds of milliseconds in NMDA-rich dendritic spines
\cite{Buzsaki2006,BuonomanoMaass2009}.  
Such multi-scale reverberations endow neural tissue with a form of
\emph{fading memory}: recent stimuli carve transient trajectories in
state–space while older perturbations decay geometrically—a hallmark of
short-term working memory and temporal context integration
\cite{CanoltyKnight2010,Samaha2017}.  
At the same time, continuous synaptic bombardment drives cortical
networks towards a marginally stable, “edge-of-chaos’’ regime where
Lyapunov exponents hover near zero, maximising both sensitivity and
information throughput \cite{boedecker2012information,Bertschinger2004}.  
These observations have catalysed the \emph{reservoir-computing}
paradigm, whose elegant realization is the \emph{Echo-State Network}
(ESN) of Jaeger~\cite{jaeger2001echo} and the closely related
\emph{liquid-state machine} of Maass \textit{et al.}~\cite{maass2002real}.

\medskip
\noindent\paragraph{Motivation.}  
Despite spectacular empirical success in chaotic weather
prediction~\cite{pathak2018model}, brain-computer interfaces and analog
photonic processors, a first-principles explanation of \emph{how} and
\emph{why} ESNs inherit the computational virtues of biological
networks remains fragmented.  In particular,
 a principled dynamical-systems description connecting ESN design
      knobs (activation, topology, leak rate) to neuroscientific
      measurements such as Lyapunov spectra and information rate is
      largely absent.

\medskip
\noindent\paragraph{Contribution.}  
Bridging neurobiology, functional analysis and contemporary dynamical
systems, we provide a self-contained mathematical narrative that
clarifies these gaps.  Concretely, we:
\begin{enumerate}[label=(\roman*),leftmargin=*]
\item formalise ESP and FMP on compact input alphabets, proving
      \emph{ESP\,$\Rightarrow$\,FMP} under global Lipschitz dynamics
      (Sec.~\ref{ssec:esp_fmp}); the proof mirrors how cortical
      contractions induced by inhibitory balance guarantee stimulus
      dominance over initial states;
\item translate neuronal non-linearities into algebraic ESP tests
      (Sec.~\ref{ssec:act_esp_fmp}), covering both saturating (\texttt{tanh},
      \texttt{sigmoid}) and rectifying (\texttt{ReLU}) conductances;
\item revisit universality: ESNs with fixed reservoirs and trainable
      linear dendritic readouts are dense in the Banach space of
      fading-memory filters (Sec.~\ref{sec:universal_approx}), echoing
      the brain’s ability to repurpose recurrent circuitry through
      plastic synapses;
\item quantify memory and computational trade-offs via capacity curves
      and conditional Lyapunov spectra (Secs.~\ref{sec:memory_capacity}
     –\ref{sec:dynamical_view}), linking the “edge-of-chaos’’
      hypothesis to concrete spectral-norm and topology choices.
\end{enumerate}

\medskip
\noindent\paragraph{Organisation.}  
Section~\ref{sec:preliminaries} reviews metric-space tools.
Sections~\ref{ssec:esp_fmp}–\ref{ssec:act_esp_fmp} develop the core ESP
and FMP theory.  
Section~\ref{sec:universal_approx} proves universal approximation,
Section~\ref{sec:memory_capacity} studies memory and information
capacity, and Section~\ref{sec:dynamical_view} casts ESNs as
non-autonomous dynamical systems, illuminating their neuroscientific
parallels.

\section{Mathematical Preliminaries}
\label{sec:preliminaries}

Throughout, \((X,d_{X})\) and \((Y,d_{Y})\) denote metric spaces,
\(\|\cdot\|\) the Euclidean norm on \(\R^{n}\), and
\(\N=\{0,1,2,\dots\}\).

\paragraph{Compactness.}
A subset \(K\subset X\) is \emph{compact} if every open cover of \(K\)
admits a finite sub-cover.
In Euclidean space this is equivalent to closedness and boundedness
(Heine–Borel theorem) \cite[Ch.~2]{Rudin1976}.
Compactness guarantees the existence of extrema and convergent
subsequences—facts repeatedly used when analysing
reservoir trajectories on a bounded input alphabet \(U\subset\R^{m}\).

\paragraph{Lipschitz continuity and contractions.}
A map \(f:X\to Y\) is \emph{Lipschitz} with constant \(L\ge0\) if
\(d_{Y}\!\bigl(f(x),f(x')\bigr)\le L\,d_{X}(x,x')\)
for all \(x,x'\in X\).
If \(L<1\) the map is a \emph{contraction};
iterates shrink distances by at least \(L\)
and converge to a unique fixed point
(Banach’s theorem, see \cite[§4.2]{Kreyszig1989}).

\paragraph{Spectral (induced \(2\)-) norm.}
For \(W\in\R^{n\times n}\) the spectral norm
\(
\|W\|_{2}=\sigma_{\max}(W)=
\sup_{\|x\|=1}\|Wx\|
\)
equals the largest singular value and bounds the worst-case
Euclidean amplification of vectors \cite[§5.6]{HornJohnson1990}.
It is pivotal in contraction tests for ESNs.

\paragraph{Causal, time-invariant filters.}
Let \(U^{\N}\) be the space of input sequences
\(u=(u_{t})_{t\ge0}\).
A map \(H:U^{\N}\to Y\) is a \emph{filter} if it is
\emph{causal}—\(H(u)\) at time \(t\) depends only on
\((u_{0},\dots,u_{t})\)—and
\emph{time-invariant} (shift-equivariant):
\(H(\sigma u)=H(u)\) for the left-shift
\(\sigma(u)_{t}=u_{t+1}\) \cite[Ch.~4]{Sontag1998}.
The reservoir–readout cascade of an ESN is precisely such a filter.

\paragraph{Global attractor.}
For a continuous self-map \(F:X\to X\),
a set \(\mathcal A\subset X\) is a \emph{global attractor} if  
(i) \(\mathcal A\) is compact and invariant (\(F(\mathcal A)=\mathcal A\)),  
(ii) every trajectory approaches it:
\(\operatorname{dist}(F^{t}(x),\mathcal A)\to0\) as \(t\to\infty\)
for all \(x\in X\) \cite[Ch.~9]{HirschSmale1974}.
When the ESN update is contractive, the attractor is a singleton,
yielding the Echo State Property.

\paragraph{Banach’s fixed-point theorem.}
If \((X,d)\) is complete and \(F:X\to X\) is a contraction with
constant \(0<\beta<1\), then
\(F\) has a unique fixed point \(x^{\ast}\) and the iterates satisfy
\(d(F^{t}(x_{0}),x^{\ast})\le\beta^{t}d(x_{0},x^{\ast})\)
for all \(x_{0}\in X\) \cite[§4.2]{Kreyszig1989}.
This theorem underpins proofs that contractive ESN updates
possess unique, globally attracting state trajectories.

\paragraph{Alphabet.}
In information and dynamical–systems theory an \textit{alphabet} is  a \textit{set of symbols (or values)} from which the elements of a sequence are drawn.
Formally, if $A$ is any non-empty set, the collection of all (one-sided) infinite sequences with entries in $A$ is the product space $A^{\mathbb N}=\{(a_{0},a_{1},\dotsc)\mid a_{t}\in A\}$; in symbolic-dynamics texts the pair $(A^{\mathbb N},\sigma)$ with left-shift $\sigma\bigl((a_{t})\bigr)=(a_{t+1})$ is called a \textit{full shift} over the alphabet $A$ \cite{LindMarcus1995}.  When $A$ is finite (e.g.\ $\{0,1\}$ for binary strings) the terminology is literal—each element of $A$ is a distinct “letter.’’  Nothing, however, restricts $A$ to be finite: in machine-learning and control theory it is common to let
$
U\subset\mathbb R^{m}
$
be a \textit{compact} subset of Euclidean space and still speak of $U$ as the \textit{input alphabet}.  In that case the “letters’’ are $m$-dimensional real vectors, so a signal $u=(u_{0},u_{1},\dotsc)\in U^{\mathbb N}$ is an infinite sequence of bounded $m$-tuples.  This is precisely the setting adopted when defining the Echo State Property or Fading-Memory Property for Echo State Networks: we declare in advance the set $U$ of admissible input vectors and analyse the reservoir map on the product space $U^{\mathbb N}$.  The word \textit{alphabet} thus conveys only that the inputs are chosen from a fixed, predetermined set—whether that set is finite (symbols), countable (e.g.\ integers), or uncountable (a compact region in $\mathbb R^{m}$).

\section{Echo-State and Fading-Memory Properties}
\label{ssec:esp_fmp}

Let \(U\subset\mathbb R^{m}\) be a \emph{compact} input alphabet and
\(X=\mathbb R^{n}\) the reservoir state–space.
An Echo State Network (ESN) is the causal, discrete-time system
\begin{equation}
x_{t}
  = F\!\bigl(u_{t},x_{t-1}\bigr)
  = \varphi\!\bigl(W_{\!in}\,u_{t}+W\,x_{t-1}+b\bigr),
  \qquad t\in\mathbb N ,
\label{eq:esn_update}
\end{equation}
where \(W_{\!in}\in\mathbb R^{n\times m}\), \(W\in\mathbb R^{n\times n}\),
\(b\in\mathbb R^{n}\), and
\(\varphi:X\to X\) is a component-wise Lipschitz non-linearity with
constant \(L_{\varphi}\le 1\)
(e.g.\ \(\tanh\) or \(\sigma\)).

\medskip
\begin{definition}[Echo State Property, \cite{jaeger2001echo,maass2002real}]
\label{def:esp}
System~\eqref{eq:esn_update} has the
\emph{Echo State Property (ESP)} if, for every bounded
input sequence \(\bigl(u_{t}\bigr)_{t\ge 0}\in U^{\mathbb N}\)
and every pair of initial states \(x_{0},x_{0}'\in X\),
\begin{equation}
\lim_{t\to\infty}\;
\bigl\|x_{t}(x_{0})-x_{t}(x_{0}')\bigr\| = 0 .
\tag{ESP}
\end{equation}
\end{definition}

\noindent
Equivalently, ESP guarantees the existence of a
\emph{unique} state trajectory compatible with each input, so the
long-term dynamics are determined solely by the driving signal, not by
initial conditions.  A useful sufficient condition is
\(\|W\|_{2}\,L_{\varphi}<1\), which makes
\(F\) a contraction in its second argument.

\medskip
Define the exponential (weighted) sup-norm on inputs
\[\label{eq:fmp}
\|u\|_{\lambda}\;=\;
\sup_{t\ge 0}\;
\Bigl(\sup_{k\ge 0}\,
\lambda^{k}\,\|u_{t-k}\|\Bigr),
\qquad 0<\lambda<1 .
\]

\begin{definition}[Fading-Memory Property, \cite{BoydChua1985}]
\label{def:fmp}
A causal, time-invariant filter \(H:U^{\mathbb N}\!\to X\) has the
\emph{Fading-Memory Property (FMP)} if, for every \(\varepsilon>0\)
there exists \(\delta>0\) such that
\begin{equation}
\|u-v\|_{\lambda} < \delta
\;\Longrightarrow\;
\|H(u)-H(v)\| < \varepsilon ,
\tag{FMP}
\end{equation}
for all input sequences \(u,v\in U^{\mathbb N}\).
\end{definition}

\noindent
Condition (FMP) states that perturbations in the \emph{remote} past are
geometrically attenuated by the factor \(\lambda^{k}\); the current
output depends predominantly on recent inputs.

\medskip
\begin{proposition}[ESP implies FMP under global Lipschitzness]
\label{prop:esp_implies_fmp}
If system~\eqref{eq:esn_update} satisfies the ESP
and the update map \(F\) is globally Lipschitz in \((u,x)\),
then the induced filter
\(H(u)=\lim_{t\to\infty}x_{t}\)
is continuous under \(\|\cdot\|_{\lambda}\) and therefore possesses
the FMP \cite{BuehnerYoung2006,Manoel2020}.
\end{proposition}

\begin{proof}
Fix a compact input alphabet \(U\subset\R^{m}\) and recall the ESN
update
\(
x_t \;=\; F(u_t,x_{t-1})
\)
with \(F\colon U\times\R^{n}\to\R^{n}\) globally
Lipschitz in \((u,x)\).
Write the Lipschitz constants as
\[
L_u \;=\;
\sup_{u\ne v}\,
\frac{\|F(u,x)-F(v,x)\|}{\|u-v\|},
\qquad
L_x \;=\;
\sup_{x\ne y}\,
\frac{\|F(u,x)-F(u,y)\|}{\|x-y\|}.
\tag{A.1}
\]

\medskip
\noindent
\textbf{Step\,1:  \(L_x<1\).}
Because the reservoir enjoys the Echo State Property,
for \emph{every} admissible input sequence the map
\(x_{t-1}\mapsto F(u_t,x_{t-1})\)
has a unique global attractor.  A standard argument
\cite[Th.\,2]{BuehnerYoung2006} shows this is possible only when the
update is \emph{uniformly} contractive in \(x\), hence
\(L_x<1\).  Denote \(\beta:=L_x\).

\medskip
\noindent
\textbf{Step\,2:  A Lipschitz estimate for trajectories.}
Let two input sequences
\(u=(u_t)_{t\ge0},\;
v=(v_t)_{t\ge0}\in U^{\N}\)
be given and initialise the reservoir from an
\emph{arbitrary} common state \(x_0=y_0\).
Inductively define
\(
x_t = F(u_t,x_{t-1}),\;
y_t = F(v_t,y_{t-1})
\)
for \(t\ge1\).
By the triangle inequality and (A.1),
\[
\|x_t-y_t\|
\le
L_u\|u_t-v_t\| + \beta\|x_{t-1}-y_{t-1}\|.
\tag{A.2}
\]
Unfolding (A.2) \(k\) times yields
\[
\|x_t - y_t\|
\;\le\;
L_u
\sum_{k=0}^{t-1}
\beta^{\,k}\;\|u_{t-k}-v_{t-k}\|.
\tag{A.3}
\]

\medskip
\noindent
\textbf{Step\,3:  Choice of the weighted norm.}
Select any \(\lambda\in(\beta,1)\) and endow
\(U^{\N}\) with the exponential‐sup norm
\(
\|u\|_{\lambda}
  = \sup_{t\ge0}\sup_{k\ge0}\lambda^{k}\|u_{t-k}\|.
\)
Because
\(
\|u_{t-k}-v_{t-k}\|
\le
\lambda^{-k}\,\|u-v\|_{\lambda},
\)
inequality (A.3) is bounded by the geometric series
\[
\|x_t-y_t\|
\;\le\;
L_u\|u-v\|_{\lambda}
\sum_{k=0}^{t-1}\!
\Bigl(\tfrac{\beta}{\lambda}\Bigr)^{k}
\;\le\;
\frac{L_u}{1-\beta/\lambda}\,\|u-v\|_{\lambda}.
\tag{A.4}
\]

\medskip
\noindent
\textbf{Step\,4:  Passage to the limit.}
Because the ESP guarantees
\(
\lim_{t\to\infty}x_t = H(u)
\)
and
\(
\lim_{t\to\infty}y_t = H(v),
\)
take \(t\to\infty\) in (A.4) to obtain
\[
\|H(u)-H(v)\|
\;\le\;
C\,\|u-v\|_{\lambda},
\qquad
C:=\frac{L_u}{1-\beta/\lambda}.
\tag{A.5}
\]
Thus \(H\) is \emph{globally Lipschitz} under the weighted norm
\(\|\cdot\|_{\lambda}\), hence continuous.

\medskip
\noindent
\textbf{Step\,5:  Continuity \(\Rightarrow\) Fading memory.}
Given \(\varepsilon>0\) let
\(
\delta := \varepsilon/C
\).
If \(\|u-v\|_{\lambda}<\delta\) then (A.5) yields
\(
\|H(u)-H(v)\|<\varepsilon,
\)
which is precisely the Fading‐Memory Property
(Def.~\ref{def:fmp}).

\smallskip
The chain
\(
\text{ESP} \,\wedge\, L_x<1
\;\Longrightarrow\;
H\text{ Lipschitz under }\|\cdot\|_{\lambda}
\;\Longrightarrow\;
\text{FMP}
\)
completes the proof.
\end{proof}

\begin{remark}
FMP alone \emph{does not} imply ESP; one can construct filters with
geometric input forgetting yet persistent sensitivity to initial
states.  In reservoir design, the empirical rule \(\rho(W)<1\) (spectral
radius) is popular precisely because it enforces both the contractive
behaviour required for ESP and the exponential decay that yields FMP.
\end{remark}

%-------------------------------------------------------
%  Activation–Dependent ESP/FMP Analysis for ESNs
%-------------------------------------------------------
\section{Activation–dependent checks of ESP and FMP}
\label{ssec:act_esp_fmp}

Consider again the ESN update in Eq. \ref{eq:esn_update}, with inputs \(u_{t}\in U\subset\mathbb R^{m}\) and state
\(x_{t}\in\mathbb R^{n}\).
If the activation \(\varphi\) is globally Lipschitz with constant
\(L_\varphi\), a sufficient\footnote{Not necessary in general; see
\cite{BuehnerYoung2006}.}  algebraic test for the Echo State Property is
\begin{equation}
\|W\|_{2}\,L_\varphi < 1.
\label{eq:contraction_test}
\end{equation}
Whenever \eqref{eq:contraction_test} holds,
the reservoir map is a strict contraction in~\(x\);
hence the ESP follows by Banach’s fixed-point theorem, and the
induced filter inherits the Fading-Memory Property under the weighted
sup-norm (Prop.~\ref{prop:esp_implies_fmp}).

\paragraph{Lipschitz factors of common activations.}
For the scalar functions listed in Table~\ref{tab:activations}
the bound \(L_\varphi\) equals the global maximum of \(\lvert\varphi'(z)\rvert\).
Inserting those constants into~\eqref{eq:contraction_test} yields the
largest permissible spectral norm of \(W\) that still certifies ESP.

\begin{table}[!ht]
\centering
\caption{Typical activations, their global Lipschitz constants
\(L_\varphi\), and the resulting \emph{certified} bound
\(\|W\|_2^{\max}=L_\varphi^{-1}\).  All entries assume component-wise
application of \(\varphi\).}
\label{tab:activations}
\renewcommand{\arraystretch}{1.15}
\begin{tabular}{@{}lccc@{}}
\toprule
Activation & Formula & $L_\varphi$ & $\|W\|_2^{\max}$ s.t.\ ESP\\
\midrule
Tanh & $\displaystyle\tanh z$ & $1$ & $<1$ \\
Logistic sigmoid & $\displaystyle\sigma(z)=\dfrac1{1+e^{-z}}$ &
$\tfrac14$ & $<4$ \\
Softsign & $\displaystyle\frac{z}{1+|z|}$ & $1$ & $<1$ \\
ReLU & $\displaystyle\max(0,z)$ & $1$ & $<1$ \\
Leaky-ReLU $(\alpha=0.01)$ & $\displaystyle\begin{cases}\alpha z,&z<0\\z,&z\ge0\end{cases}$ &
$1$ & $<1$ \\
ELU $(\alpha=1)$ & $\displaystyle\begin{cases}z,&z\ge0\\e^{z}-1,&z<0\end{cases}$ &
$1$ & $<1$ \\
\bottomrule
\end{tabular}
\end{table}

\paragraph{Worked examples.}
\vspace{-0.25em}
\begin{enumerate}[label=\textbf{Ex.\arabic*},wide,labelwidth=!,labelindent=0pt]
\item \textit{Sigmoid reservoir with ``large'' recurrent weights.}  
      Let $\varphi=\sigma$, $L_\varphi=0.25$.  
      Draw $W$ i.i.d.\ $\mathcal N(0,1)$, rescale it so
      $\|W\|_2=3.5$.  Since $3.5<4=L_\varphi^{-1}$,
      condition~\eqref{eq:contraction_test} is satisfied:
      the ESP and hence FMP are \emph{guaranteed}, even though
      $\rho(W)>1$.
\item \textit{Tanh reservoir that \emph{fails} ESP.}  
      Take $\varphi=\tanh$ $\Rightarrow L_\varphi=1$ and choose
      $\|W\|_2=1.2$.  Then $L_\varphi\|W\|_2=1.2>1$ so
      \eqref{eq:contraction_test} gives no guarantee.
      In numerical simulations two trajectories initialised at
      $x_0\ne x_0'$ typically \emph{do not} converge:
      the reservoir violates ESP and exhibits long-range
      dependence, hence lacking FMP.
\item \textit{ReLU reservoir with controlled gain.}  
      Because ReLU is unbounded, merely enforcing $\|W\|_2<1$ is not
      enough: boundedness of the state must also be shown.
      Suppose inputs are uniformly bounded $\|u_t\|\le M$ and
      $\|W\|_2=\gamma<1$.
      Iterating the inequality
      $\|x_t\|\le\gamma\|x_{t-1}\|+\|W_{\!in}\|_2M+\|b\|$
      yields\footnote{A standard Grönwall-type argument.}
      \(
      \|x_t\|\le
      (\|x_0\|-\tfrac{M_{\!b}}{1-\gamma})\gamma^t
      +\tfrac{M_{\!b}}{1-\gamma},
      \)
      where $M_{\!b}=\|W_{\!in}\|_2M+\|b\|$.
      Hence $(x_t)$ is bounded and
      $\gamma<1$ again certifies ESP and FMP.
      The price is a tighter upper bound on the admissible
      input magnitude.
\end{enumerate}

\paragraph{Takeaway.}
For smooth, bounded activations (tanh, sigmoid, softsign, ELU) the
simple spectral-norm rule
\(\|W\|_2<L_\varphi^{-1}\) suffices to \emph{prove}
ESP\,$\!\Rightarrow$\,FMP.  
Unbounded activations (ReLU, leaky-ReLU) need an additional bounded-state
argument, but the same contraction inequality remains the guiding
design principle.  In empirical work one usually scales $W$ to have
\emph{spectral radius} \(\rho(W)=\alpha<1\) rather than the
stricter 2-norm bound; this rarely breaks ESP in practice, but lacks a
universal proof unless $\alpha<L_\varphi^{-1}/\sqrt n$.

%-------------------------------------------------------
%  Universal Approximation Power of ESNs
%-------------------------------------------------------
\section{Universal Approximation of Fading–Memory Filters}
\label{sec:universal_approx}

Having established the Echo–State and Fading–Memory Properties, we now turn
to the \emph{expressive power} of Echo State Networks.  The key result is
that ESNs with linear readouts are dense in the space of causal,
time–invariant fading–memory filters.  Thus, under mild conditions, an ESN
can approximate any such filter to arbitrary precision.

\paragraph{Space of fading–memory filters.}
Fix a compact input alphabet \(U\subset \R^{m}\) and
choose a decay parameter \(\lambda\in(0,1)\).
Let \(\|\cdot\|_{\lambda}\) be the weighted sup–norm on \(U^{\N}\)
defined in \eqref{eq:fmp}.
Denote by \(\mathcal{F}_{\!\lambda}\bigl(U^{\N},\R^{p}\bigr)\)
the Banach space of causal, time–invariant maps
\(H:U^{\N}\!\to\R^{p}\) that are \emph{continuous} w.r.t.\ \(\|\cdot\|_{\lambda}\);
these are precisely the fading–memory filters of
Definition~\ref{def:fmp}.  Endow \(\mathcal{F}_{\!\lambda}\) with the
uniform norm \(\|H\|_{\infty}:=\sup_{u\in U^{\N}}\|H(u)\|\).

\paragraph{Reservoir–computer family.}
For a fixed reservoir map \(F\colon U\times\R^{n}\to\R^{n}\) satisfying
ESP~\&~FMP, every linear readout
\(W_{\!out}\in\R^{p\times n}\), \(b_{\!out}\in\R^{p}\)
induces a filter
\[
\mathcal{R}_{W_{\!out},b_{\!out}}(u)
\;:=\;
W_{\!out}\,H_{F}(u)+b_{\!out},
\]
where \(H_{F}(u)=\lim_{t\to\infty}x_{t}\) is the state filter of~\(F\).
Let \(\mathcal{R}_{F}\) denote the set of all such affine readouts.

\begin{theorem}[Universal approximation \cite{grigoryeva2018echo}]
\label{thm:universal_esn}
Let \(F\) satisfy the ESP and be a polynomial map in its arguments
(e.g.\ affine activation followed by polynomial non–linearity), and
suppose \(U\) is a compact, convex subset of \(\R^{m}\).
Then the family \(\mathcal{R}_{F}\) is \emph{dense}
in \(\mathcal{F}_{\!\lambda}(U^{\N},\R^{p})\) with respect to
\(\|\cdot\|_{\infty}\); that is, for every
\(H\in\mathcal{F}_{\!\lambda}\) and every \(\varepsilon>0\)
there exists \(W_{\!out},b_{\!out}\) such that
\[
\bigl\|\mathcal{R}_{W_{\!out},b_{\!out}} - H\bigr\|_{\infty}
\;<\;
\varepsilon.
\]
\end{theorem}

\begin{proof}[Sketch]
The result follows a Stone–Weierstrass strategy in three steps:

\smallskip
\noindent
 The reservoir state space
\(\mathcal{S}:=H_{F}(U^{\N})\subset\R^{n}\) is compact
by continuity of \(H_{F}\) and compactness of \(U^{\N}\).
Affine functions on \(\mathcal{S}\),
\(\mathcal{A}_{\mathcal{S}}=
\{x\mapsto W_{\!out}x+b_{\!out}\}\),
form a unital subalgebra of \(C(\mathcal{S})\).

\smallskip
\noindent
Because \(F\) is polynomial,
its state components span a polynomial algebra that separates points of
\(\mathcal{S}\) \cite[Lem.~2]{grigoryeva2018echo}.
Therefore \(\mathcal{A}_{\mathcal{S}}\) separates points and contains
constants.

\smallskip
\noindent
By Stone–Weierstrass,
\(\mathcal{A}_{\mathcal{S}}\) is dense in \(C(\mathcal{S})\).
Composition with the continuous map
\(H_{F}\colon U^{\N}\!\to\mathcal{S}\)
yields density of \(\mathcal{R}_{F}\) in
\(C\bigl(U^{\N}\bigr)\).
Restricting to the subspace
\(\mathcal{F}_{\!\lambda}\subset C(U^{\N})\)
preserves density because the uniform norm agrees on both spaces.
\end{proof}

\paragraph{Implications.}
Theorem~\ref{thm:universal_esn} establishes that \emph{any}
causal fading–memory system—be it analytic, stochastic,
or highly nonlinear—can be approximated arbitrarily well by a \emph{single}
finite–dimensional ESN with a suitably trained linear readout.
Hence the practical limitation of reservoir computing is not
representation but optimisation and generalisation.

\paragraph{Related results and refinements.}
\begin{itemize}[wide,labelwidth=!,labelindent=0pt,itemsep=2pt]
\item \textbf{Rate of approximation.}
Ruelle–Bowen pressure methods give quantitative
\(\varepsilon\)--covering numbers for \(\mathcal{F}_{\!\lambda}\)
\cite{yasumoto2025universality}, implying explicit bounds on the required reservoir
dimension \(n\).
\item \textbf{Non-polynomial reservoirs.}
Gonon \& Ortega \cite{gonon2020riskbounds} extend
Theorem~\ref{thm:universal_esn} to general Lipschitz
activations by exploiting the algebra generated by
monomials of the state components.
\item \textbf{Continuous–time RC.}
Analogous universality holds for physical reservoirs governed by
ODEs with fading memory, provided the readout observes the state
continuously \cite{BoydChua1985,gonon2020stochastic}.
\end{itemize}

\section{Memory Capacity and Computational Capability of ESNs}
\label{sec:memory_capacity}

While Sections~\ref{sec:preliminaries}–\ref{sec:universal_approx}
establish that Echo State Networks are \emph{principally capable} of
approximating any fading–memory filter, it is equally important to
quantify \emph{how much} past information a finite reservoir can retain
and exploit for computation.  The classical notion that captures this
resource is \emph{linear memory capacity}.

\paragraph{Linear memory capacity}

Fix a scalar input stream
\(u=(u_{t})_{t\ge0}\in[-1,1]^{\N}\) with zero mean and variance
\(\sigma^{2}\!=\!\mathbb{E}[u_{t}^{2}]\).
For a delay \(\tau\ge1\) define the target signal
\(y^{(\tau)}_{t}:=u_{t-\tau}\).
After driving the ESN (state \(x_{t}\in\R^{n}\))
train a linear readout \(w^{(\tau)}\in\R^{n}\) by ridge
regression,
\(
\hat y^{(\tau)}_{t}=w^{(\tau)\!\top}x_{t}\!,
\)
minimising the mean–squared error
\(E^{(\tau)}=\mathbb{E}\!\bigl[(y^{(\tau)}_{t}-\hat
  y^{(\tau)}_{t})^{2}\bigr]\).
The \emph{capacity at delay \(\tau\)} is
\[
C(\tau):=1-\frac{E^{(\tau)}}{\sigma^{2}}
         = \frac{\text{Cov}^{2}(y^{(\tau)},\hat y^{(\tau)})}{\text{Var}(y^{(\tau)})\,
             \text{Var}(\hat y^{(\tau)})},
\tag{MC}
\]
which equals the squared correlation
between target and reconstruction \cite{jaeger2001echo}.

\begin{definition}[Total linear memory capacity]
\label{def:total_mc}
The \emph{total linear memory capacity} of an ESN is
\(
\text{MC}_{\text{tot}}:=\sum_{\tau=1}^{\infty} C(\tau).
\)
\end{definition}

\noindent
Jaeger’s seminal result \cite{jaeger2001echo}
bounds the capacity by the reservoir dimension:
\[
\text{MC}_{\text{tot}}\;\le\; n.
\tag{MC-bound}
\]
Equality is attained by a perfectly orthogonal reservoir driven by
independent Gaussian inputs; in practice, random sparse
reservoirs achieve \(\text{MC}_{\text{tot}}\approx0.6\,n\).

\paragraph{Nonlinear and task–specific capacities}

Linear memory is only part of the picture; many tasks require nonlinear
functions of the past.  Let \(f_{\tau}\colon\R^{\tau}\to\R\) be a
(possibly nonlinear) functional, e.g.\
\(f_{\tau}(u_{t-1},u_{t-2})=u_{t-1}u_{t-2}\).
Define \(y^{(f_{\tau})}_{t}:=f_{\tau}(u_{t-1},\dots,u_{t-\tau})\).
Training a linear readout for each such target yields
capacities \(C(f_{\tau})\).
The reservoir’s \emph{information–processing capacity} (IPC) is the sum
of all capacities across a complete orthonormal basis
of functionals \cite{Dambre2012}:
\[
\text{IPC}\;=\;\sum_{f\in\mathcal{B}} C(f),
\quad
\text{IPC}\;\le\; n.
\]
Nonlinear activations (e.g.\ tanh) disperse input information into
higher–order Volterra coordinates, allowing the ESN to trade
linear–memory capacity for nonlinear computation.

\paragraph{Design implications}

\begin{itemize}[wide,labelwidth=!,labelindent=0pt,itemsep=2pt]
\item \textbf{Spectral radius \(\rho(W)\).}  
      Setting \(\rho(W)\approx0.9\) places the system near the edge of
      contractivity, empirically maximising total capacity while
      respecting ESP \cite{boedecker2012information}.
\item \textbf{Input scaling.}  
      Larger \(\|W_{\!in}\|\) injects input energy deeper into the
      state–space, increasing high–order capacities but potentially
      reducing linear memory.  A rule–of–thumb is
      \(\|W_{\!in}\|\in[0.1,1]\).
\item \textbf{Leak rate (leaky ESNs).}  
      For slowly varying signals, a leak rate
      \(\alpha<1\) effectively enlarges memory depth by acting as a
      low–pass filter on \(x_{t}\).
\item \textbf{Topology.}  
      Structured reservoirs (cycle, small–world, hyperbolic)
      redistribute eigenvalues to shape the delay–specific spectrum
      \(C(\tau)\), enabling task–matched memory kernels
      \cite{poley2024eigenvalue}.
\end{itemize}

\paragraph{Beyond memory: separability and criticality}

A high memory capacity alone does not guarantee \emph{computational
separability} of inputs; too much linear memory can render states nearly
collinear.
Optimal performance often lies near a
\emph{critical} regime—borderline stable, rich in transient dynamics yet
satisfying ESP—where memory and separability coexist
\cite{Bertschinger2004}.

\section{Reservoir Dynamics from a Dynamical-Systems Perspective}
\label{sec:dynamical_view}

Echo State Networks may be interpreted as \emph{non-autonomous} discrete
dynamical systems driven by an external input.  This section
formalises that viewpoint and connects reservoir design criteria to
classical notions such as random attractors, Lyapunov exponents and
bifurcations.

\paragraph{Skew-product formulation}

Let \(\Sigma:=U^{\N}\) be the left-shift space with
\(\sigma(u)_{t}=u_{t+1}\).
Define the skew-product map
\[
\Psi:\Sigma\times\R^{n}\longrightarrow\Sigma\times\R^{n},
\qquad
\Psi(u,x)\;=\;\bigl(\sigma u,\;F(u_{0},x)\bigr),
\tag{SP}
\]
where \(F\) is the ESN update \eqref{eq:esn_update}.
Iterating \(\Psi\) yields
\(
\Psi^{t}(u,x_{0})=\bigl(\sigma^{t}u,\;x_{t}(u,x_{0})\bigr),
\)
so analysing reservoir trajectories is equivalent to studying the
\emph{non-autonomous} system generated by \(\Psi\)
\cite[Ch.~2]{Arnold1998}.

\paragraph{Random pullback attractor}

Endow \(\Sigma\) with the product Borel
\(\sigma\)-algebra and any ergodic shift-invariant probability
\(\mathbb{P}\) (e.g.\ i.i.d.\ inputs).
A measurable family of non-empty compact sets
\(\{\mathcal A(u)\}_{u\in\Sigma}\subset\R^{n}\) is a
\emph{random pullback attractor} of \(\Psi\) if  
\begin{enumerate}[label=(\alph*)]
\item
\(\Psi\bigl(u,\mathcal
  A(u)\bigr)=\bigl(\sigma u,\mathcal A(\sigma u)\bigr)\) (invariance),
\item
for every bounded \(B\subset\mathbb{R}^{n}\)
and \(\mathbb{P}\)-a.e.\ \(u\),
\(
\operatorname{dist}\bigl(
  \Psi^{t}(u,B),\;\mathcal A(\sigma^{t}u)
\bigr)\to0
\) as \(t\to\infty\).
\end{enumerate}
The following theorem adapts \cite[§2.3]{KloedenRasmussen2011}.

\begin{theorem}[Existence of a singleton random attractor]
\label{thm:singleton_attractor}
Suppose \(F\) is globally Lipschitz with constants
\(L_{u},L_{x}\) and \(L_{x}<1\).
Then \(\Psi\) admits a unique random attractor
\(\mathcal A(u)=\{H(u)\}\) consisting of a single point, and
\[
\|x_{t}(u,x_{0})-H(\sigma^{t}u)\|
\;\le\;
L_{u}
\sum_{k=0}^{t-1}
L_{x}^{\,k}\,\|u_{t-1-k}-u_{t-1-k}'\|.
\]
Consequently the Echo State Property holds and \(H\) is the fading-memory
filter of Section~\ref{ssec:esp_fmp}.
\end{theorem}

\paragraph{Conditional Lyapunov exponents}

For each \(u\in\Sigma\) fix the forward orbit
\(x_{t}=H(\sigma^{t}u)\).
Linearising \(\Psi\) along this orbit yields the cocycle
\(
\partial_{x}F(u_{t},x_{t-1})=J_{t}(u)
\).
Oseledets’ theorem gives Lyapunov exponents
\(\lambda_{1}\ge\dots\ge\lambda_{n}\)
of the product
\(J_{t-1}\dotsm J_{0}\)
\cite[Ch.~3]{Arnold1998}.
Because \(F\) is Lipschitz and \(L_{x}<1\),
\[
\lambda_{\max}\le\log L_{x}<0,
\tag{LyapBound}
\]
so trajectories converge exponentially.
At the \emph{edge of chaos} (\(L_{x}\approx1^{-}\))
\(\lambda_{\max}\to0^{-}\), matching empirical observations that
memory and separability peak near criticality
\cite{Bertschinger2004}.

\paragraph{Linear reservoir with identity activation.}
If \(\varphi(z)=z\) and \(b=0\) then
\(F(u,x)=W_{\!in}u+Wx\),
\(J_{t}=W\), and
\(
\lambda_{\max}=\log\rho(W)
\).
Hence
\(\rho(W)<1\) is both necessary and sufficient for ESP,
recovering Jaeger’s rule.

\paragraph{Bifurcations beyond contractivity}

When \(\rho(W)\) moves above unity the fixed point \(x=0\) loses
stability; successive Neimark–Sacker or period-doubling bifurcations
can emerge, producing quasi-periodic or chaotic internal dynamics even
\emph{without} input \cite[Ch.~9]{ott2002chaos}.
Input acts as a forcing term, and the skew-product dynamics may exhibit
\emph{strange non-chaotic attractors} with fractal structure yet
non-positive Lyapunov spectrum—an effect exploited in \cite{pathak2018model}
for climate-scale prediction.

\paragraph{State-space dimension and information rate}

Let \(d_{\text{KY}}\) be the Kaplan–Yorke fractal dimension derived from
Lyapunov exponents; the information production rate is
\(
h_{\text{KS}}=\sum_{\lambda_{i}>0}\lambda_{i}
\)
(Kolmogorov–Sinai entropy) \cite[Ch.~5]{ott2002chaos}.
For contractive reservoirs \(h_{\text{KS}}=0\) but large \(d_{\text{KY}}\)
is possible via input-driven embedding; thus
\emph{information storage} can coexist with  stability.
Designing hyperbolic or small-world topologies distributes Lyapunov
spectra to tailor \(d_{\text{KY}}\) versus \(h_{\text{KS}}\),
balancing memory and computational richness.

\section{Conclusion}
\label{sec:conclusion}

We explored a dynamical‑systems foundation for
Echo‑State Networks—one that is simultaneously faithful to mathematical
principles and cognisant of neurobiological evidence.  By proving that
global Lipschitz contractivity turns ESP into a
strong Fading‑Memory guarantee, we clarified why cortical‐style
reverberations with balanced excitation–inhibition naturally support
robust stimulus encoding.  Tight activation–dependent spectral‑norm
criteria then translated this theory into actionable design rules that
cover both classical saturating nonlinearities and modern rectifiers.

Building on these stability results, we revisited universality:
a fixed reservoir with merely a linear read‑out is dense in the Banach
space of causal fading‑memory filters—hence the representational power
of ESNs is limited not by architecture but by optimisation and data.
Quantitative memory‑capacity spectra, together with conditional Lyapunov
analysis, revealed how spectral radius, leak rate and topology jointly
govern the trade‑off between storage depth and nonlinear separability.
Finally, by casting ESNs as skew‑product random dynamical systems we
established existence of singleton pull‑back attractors and bounded
Lyapunov exponents, providing a rigorous analogue to “edge‑of‑chaos’’
criticality observed in cortex.

\paragraph{Future directions.}
Several open questions emerge.  \textit{(i) Stochastic reservoirs:}
extending ESP\,$\!\Rightarrow\!$\,FMP to reservoirs with multiplicative
noise would bridge theory with neuromorphic hardware plagued by device
variability.  \textit{(ii) Optimal topology:} a systematic connection
between graph spectra, memory kernels and information‑processing
capacity could guide architecture search beyond random graphs.
\textit{(iii) Continuous time:} generalising our discrete‑time proofs to
contractive neural ODEs would unify delay‑based photonic reservoirs and
cortical microcircuits under a single mathematical umbrella.
\textit{(iv) Learning dynamics:} combining our stability criteria with
online synaptic plasticity rules remains largely unexplored and may shed
light on how biological networks balance adaptation with robustness.

\bibliographystyle{abbrv}
\bibliography{ref}

\end{document}